\documentclass[twoside]{article}

\usepackage[preprint]{aistats2026}
\usepackage[round]{natbib}

\usepackage{url}
\usepackage{amsmath}
\usepackage{amsfonts}
\usepackage{amssymb}
\usepackage{amsthm}
\usepackage{graphicx}
\usepackage{tikz}
\usepackage{adjustbox}
\usepackage{booktabs}
\usepackage{algorithm}
\usepackage{algorithmic}

\allowdisplaybreaks

\newtheorem{theorem}{Theorem}
\newtheorem{proposition}{Proposition}

\definecolor{myred}{RGB}{214, 8, 59}
\definecolor{myblue}{RGB}{0, 115, 207}
\definecolor{mygreen}{RGB}{110,180,63}

\definecolor{myrgb}{RGB}{108,101,110}
\definecolor{mygb}{RGB}{55, 148, 135}
\definecolor{myrb}{RGB}{107, 62, 133}

\begin{document}

%
\runningtitle{Rethinking Inter-LoRA Orthogonality in Adapter Merging}

%

\twocolumn[

\aistatstitle{Rethinking Inter-LoRA Orthogonality in Adapter Merging:\\ Insights from Orthogonal Monte Carlo Dropout}

\aistatsauthor{ Andi Zhang$^1$~~~~Xuan Ding$^2$~~~~Haofan Wang$^3$~~~~Steven McDonagh$^4$~~~~Samuel Kaski$^{1,5}$ }

\aistatsaddress{ $^1$University of Manchester~~~~$^2$The Chinese University of Hong Kong, Shenzhen\\$^3$InstantX Team~~~~$^4$University of Edinburgh~~~~$^5$Aalto University } ]

\begin{abstract}
We propose Orthogonal Monte Carlo Dropout, a mechanism that enforces strict orthogonality when combining sparse semantic vectors without extra time complexity. Low-Rank Adaptation (LoRA), a popular fine-tuning method for large models, typically trains a module to represent a specific concept such as an object or a style. When multiple LoRA modules are merged, for example to generate an object in a particular style, their outputs (semantic vectors) may interfere with each other. Our method guarantees that merged LoRA modules remain orthogonal and thus free from direct interference. However, empirical analysis reveals that such orthogonality does not lead to the semantic disentanglement highlighted in prior work on compositional adaptation. This finding suggests that inter-LoRA orthogonality alone may be insufficient for achieving true semantic compositionality, prompting a re-examination of its role in adapter merging. Code can be found at \url{https://github.com/andiac/Orth_MC_Dropout}
\end{abstract}

\section{INTRODUCTION}
\label{sec:intro}
Low-Rank Adaptation (LoRA) \citep{hu2022lora} has become a popular method for adapting large-scale models efficiently by injecting trainable low-rank matrices while keeping pre-trained weights frozen. Beyond single-task adaptation, practitioners often merge multiple LoRA modules (hereafter referred to simply as \textbf{LoRAs}) \citep{huang2023lorahub, zhong2024multi, wang2024lora, prabhakar2024lora, po2024orthogonal} --- e.g., combining an object identity with a visual style --- for compositional generation. However, prior studies \citep{shah2023ziplora} assume that such merging may suffer from interference when LoRA modules are highly aligned, limiting semantic controllability.

Existing approaches such as Zip-LoRA \citep{shah2023ziplora} attempt to address this issue by penalizing similarity between LoRA weights, but they neither enforce strict orthogonality nor avoid the need for additional fine-tuning. To overcome these limitations, we propose \textbf{Orthogonal Monte Carlo (MC) Dropout}, a simple yet theoretically grounded mechanism that enforces orthogonality among LoRAs without additional computational overhead. Our analysis proves that the merged updates remain orthogonal at runtime, effectively eliminating direct interference.

Interestingly, experiments reveal that strict orthogonality alone does not yield the semantic disentanglement assumed in earlier compositional adaptation work \citep{shah2023ziplora}. This finding highlights a gap between structural orthogonality and true semantic compositionality, motivating a rethinking of the role of inter-LoRA orthogonality and the wider challenge of successful adapter merging.

\begin{figure*}[thb]
  \begin{center}
      \input{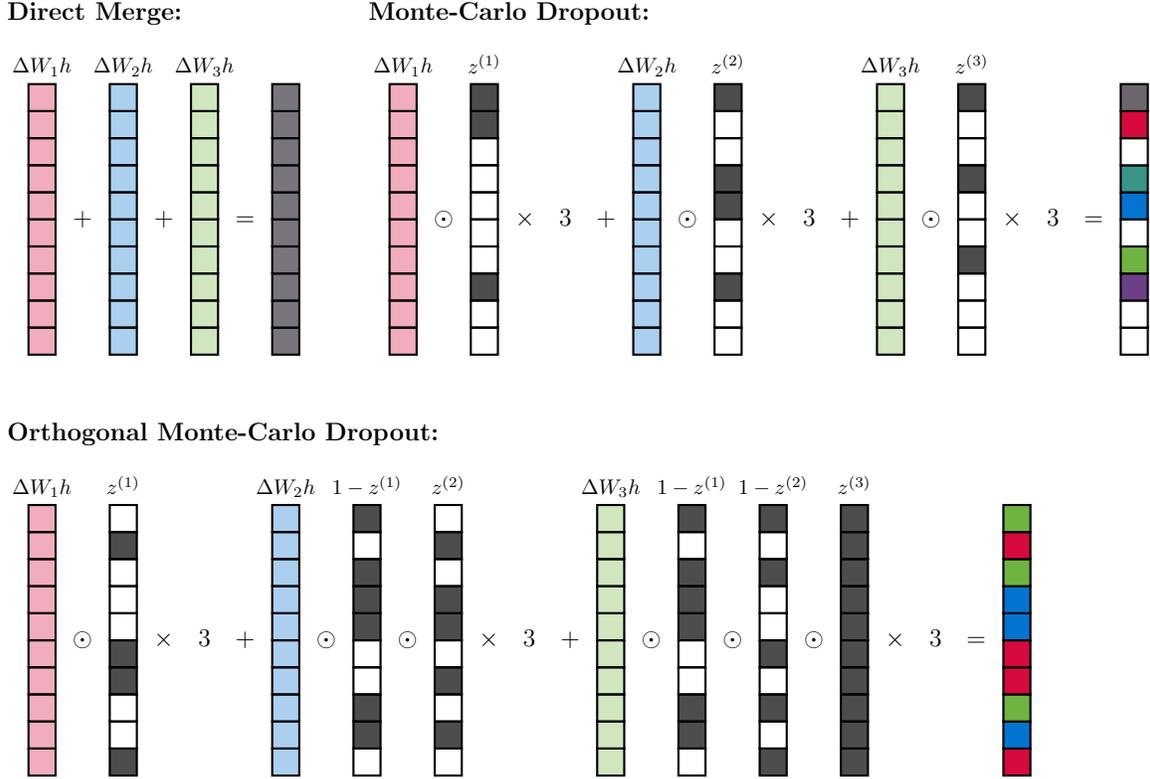}
  
      
  \end{center}

  \caption{Comparison of LoRA merging strategies. In this example, the input is $h$, and the outputs of the LoRA modules are $\Delta W_1 h, \Delta W_2 h$, and $\Delta W_3 h$. The masks $z^{(1)}, z^{(2)}, z^{(3)}$ are generated according to different merging strategies; in the figure, dark gray indicates 1 and white indicates 0. For simplicity, we set $w_i = 1$ and $p_i = 1/3$ for $i \in {1,2,3}$. \textbf{(Top Left)} Direct Merge simply sums the LoRA updates, which may lead to interference when updates are aligned. \textbf{(Top Right)} Monte Carlo Dropout merge (Algorithm~\ref{alg:dropout} and Algorithm~\ref{alg:merging}). \textbf{(Bottom)} Orthogonal Monte Carlo Dropout (Algorithm~\ref{alg:orthdropout} and Algorithm~\ref{alg:merging}), which enforces orthogonality among LoRA contributions and ensures no direct overlap.}
  \label{fig:figure}
  \end{figure*}

\noindent\textbf{Our main contributions are as follows:}
\begin{itemize}
    \item We propose \textbf{Orthogonal Monte Carlo Dropout}, a novel mechanism that enforces strict orthogonality when merging LoRAs, without incurring extra computational cost.
    \item We provide \textbf{theoretical guarantees} of runtime orthogonality.
    \item We empirically demonstrate the \textbf{redundancy} of LoRAs, which enables our method to be applied to any pre-trained LoRA from the community. 
    \item Through empirical analysis, we demonstrate that \textbf{orthogonality does not necessarily imply semantic disentanglement or compositionality}, challenging a key assumption in prior compositional adaptation work and prompting a rethinking of inter-LoRA orthogonality in adapter merging.
\end{itemize}

\section{PRELIMINARIES}

\subsection{Monte-Carlo Dropout}
Monte-Carlo (MC) Dropout \citep{gal2016dropout} leverages dropout \citep{srivastava2014dropout}, a regularization method that randomly drops units in a neural network. In MC Dropout, dropout is applied during both training and inference stages. By performing multiple forward passes through the network with dropout enabled, it generates a distribution of outputs. This distribution allows for the calculation of predictive uncertainty, providing insights into the confidence of the model's predictions. 

Let $W \in \mathbb{R}^{d_\text{out} \times d_\text{in}}$ be the parameters of a linear layer, $p$ be the dropout rate, and $h \in \mathbb{R}^{d_\text{in}}$ be the input to the layer. We define a dropout mask $z = [z_1, \dots, z_{d_\text{out}}]^T$, where each $z_i \sim \text{Ber}(1-p)$ for $i \in \{1, \dots, d_\text{out}\}$. Then, the feed-forward function $f(h)$ can be expressed as:
\begin{equation}
\label{eq:monte}
    f(h) = \frac{(Wh) \odot z}{1-p},
\end{equation}
where $\odot$ denotes element-wise multiplication. Note that the denominator $1-p$ rescales the output to compensate for the expected reduction in magnitude caused by dropout. 

In this work, we use MC Dropout not for uncertainty estimation, but for its redundancy to maintain semantic representations even when most of the neurons are masked (Section~\ref{sec:redundancy}).

\subsection{Low-Rank Adaptation (LoRA)}
Low-Rank Adaptation \citep{hu2022lora} (LoRA) is a technique that involves freezing the pre-trained model weights and injecting trainable rank decomposition matrices into each layer of the attention mechanism. This approach significantly reduces the number of trainable parameters required for downstream tasks. 

Let $W$ represent the pre-trained parameters of a linear layer within an attention mechanism, and let $h$ be the input to this linear layer. A pair of LoRA parameters, $\Delta W_a$ and $\Delta W_b$ (the trainable rank decomposition matrices), operate as follows:
\[f(h) = Wh + \Delta W_b \Delta W_a h .\]
Since the multiplication of $\Delta W_a$ and $\Delta W_b$ forms a low-rank matrix, we define $\Delta W = \Delta W_b \Delta W_a$. Thus, the operation simplifies to:
\[f(h) = Wh + \Delta W h.\]
During the training stage, $W$ is kept fixed while $\Delta W$ is tuned to maximize the likelihood on the fine-tuning dataset. 

\subsection{Merging LoRAs}
A common method to merge two LoRAs, $\Delta W_1$ and $\Delta W_2$, is through a weighted sum:
\begin{equation}
\label{eq:directmerge}
f(h) = Wh + w_1 \Delta W_1 h + w_2 \Delta W_2 h,
\end{equation}
where $w_1$ and $w_2$ are the respective weights. The work of \citet{shah2023ziplora} indicates that highly aligned LoRA weights merge poorly. This can be understood straightforwardly: if the vectors $\Delta W_1 h$ and $\Delta W_2 h$ point in similar directions, the addition operation will cause them to interfere with each other, ultimately diminishing the effectiveness of the sum.

\section{ORTHOGONAL MONTE-CARLO DROPOUT}
\label{sec:orthmcdropout}
\subsection{Definitions}
We propose \textbf{Orthogonal Monte-Carlo (MC) Dropout} to mitigate interference between LoRAs. Let $W$ be the pre-trained parameters, and $\Delta W_1$ and $\Delta W_2$ be two LoRAs with corresponding dropout rates $p_1$ and $p_2$, where $1-p_1 + 1-p_2 \leq 1$. The feed-forward function $f(h)$ is defined as:
\[
\resizebox{\linewidth}{!}{$
f(h)= Wh + \frac{w_1 z^{(1)}\odot (\Delta W_1 h)}{1-p_1} + \frac{w_2 (\mathbf{1} - z^{(1)})\odot z^{(2)}\odot (\Delta W_2 h)}{1-p_2},
$}
\]
where $z^{(1)}_i{\sim}Ber(1-p_1)$, $z^{(2)}_i{\sim}Ber((1- p_2) / p_1)$ for $i\in \{1, \dots , d_\text{out}\}$. By construction, interference between the two LoRAs is eliminated, since the inner product of $z^{(1)} \odot (\Delta W_1 h)$ and $(\mathbf{1} - z^{(1)}) \odot z^{(2)} \odot (\Delta W_2 h)$ is zero. For consistency, note that $(1 - z^{(1)}_i) \cdot z^{(2)}_i \sim \text{Ber}(1 - p_2)$.


\noindent This idea can be extended to the case of merging $k$ LoRAs with dropout probabilities $p_1, \dots, p_k$, under the constraint that $\sum_{j=1}^k (1 - p_j) \leq 1$. For notational simplicity, we assume $w_1 = w_2 = \cdots = w_k = 1$ throughout the following derivation. Under this setting, the feed-forward function $f(h)$ is defined as:
\[
f(h) = Wh + \sum_{j=1}^{k} \frac{\left(\prod_{l=1}^{j-1} (\mathbf{1} - z^{(l)}) \right) \odot z^{(j)} \odot (\Delta W_j h)}{1-p_j},
\]
where $z^{(1)}_i \sim \text{Ber}(1-p_1)$ and each $$z^{(j)}_i \sim \text{Ber}\left((1-p_j)/\left(\left(\sum_{l=1}^{j-1} p_l\right) - (j-2)\right)\right)$$ for $i \in \{1, \dots, d_\text{out}\}$, $j \in \{2, \dots, k\}$. 

\subsection{Consistency and Orthogonality}
In this subsection, we demonstrate the \textbf{consistency} and \textbf{orthogonality} of our Orthogonal MC Dropout.
\begin{theorem}{\textbf{(Consistency)}}
    Let the random mask 
    \begin{equation}
      \label{eq:mask}
      m^{(j)} = \left( \prod_{l=1}^{j-1} (\mathbf{1} - z^{(l)}) \right) \odot z^{(j)}
    \end{equation}, then $m^{(j)}_i \sim \text{Ber}(1-p_j)$.
\end{theorem}
\begin{proof}
Let $i \in \{1, \dots, d_\text{out}\}$, $j \in \{1, \dots, k\}$,
    \begin{align*}
        &\mathbb{P}(m^{(j)}_i = 1)\\ 
        =& \mathbb{P}\left(\left( \prod_{l=1}^{j-1} (1 - z^{(l)}_i) \right) \cdot z^{(j)}_i = 1\right)\\
        =& \mathbb{P}(z^{(1)}_i=0, \dots z^{(j-1)}_i=0, z^{(j)}_i = 1)\\
        =& \mathbb{P}(z^{(1)}_i=0)\cdots \mathbb{P}(z^{(j-1)}_i=0)\mathbb{P}(z^{(j)}_i = 1)\\
        =& p_1 \left( 1- \frac{1-p_2}{p_1} \right) \left( 1-\frac{1-p_3}{p_1 + p_2 -1} \right)\cdots\\
        &\left( 1 - \frac{1-p_{j-1}}{\left(\sum_{l=1}^{j-2}p_l\right)- (j-3)} \right) \frac{1-p_j}{\left(\sum_{l=1}^{j-1}p_l\right)-(j-2)}\\
        =&p_1\frac{p_1 + p_2 - 1}{p_1}\cdot\frac{p_1+p_2+p_3-2}{p_1+p_2-1}\cdots\\
        &\frac{\left(\sum_{l=1}^{j-1}p_l\right)-(j-2)}{\left(\sum_{l=1}^{j-2}p_l\right)- (j-3)}\cdot \frac{1-p_j}{\left(\sum_{l=1}^{j-1}p_l\right)-(j-2)}\\
        =& 1- p_j
    \end{align*}
\end{proof}
\noindent To facilitate the proof of orthogonality, we begin with the following proposition:
\begin{proposition}
    Let $m, n \in \{0, 1\}^d$ for some dimension $d$, such that $m^T n = 0$. Then for any $a, b\in \mathbb{R}^d$, we have $(m\odot a)^T(n\odot b) = 0$.
\end{proposition}
\begin{proof}
As $m$ and $n$ only take values 0 or 1, $m_i n_i = 0$ or $1$ for all $i$. 
Given $m^T n = m_1 n_1 + \dots + m_d n_d = 0$, we must have $m_i n_i = 0$ for all $i$ (if any $m_i n_i = 1$, then $m^T n > 0$). Therefore:
\begin{align*}
(m\odot a)^T(n\odot b) &= m_1a_1n_1b_1 + \dots  m_da_dn_db_d\\&= 0 + \dots + 0 = 0
\end{align*}
\end{proof}
\noindent Then, we prove orthogonality:
\begin{theorem}{\textbf{(Orthogonality)}}
For any \( i, j \in \{1, \dots, k\} \) with \( i < j \), the vectors 
\[ m^{(j)} \odot (\Delta W_j h) = \left( \prod_{l=1}^{j-1} (\mathbf{1} - z^{(l)}) \right) \odot z^{(j)} \odot (\Delta W_j h) \]
and
\[ m^{(i)} \odot (\Delta W_i h) =  \left( \prod_{l=1}^{i-1} (\mathbf{1} - z^{(l)}) \right) \odot z^{(i)} \odot (\Delta W_i h) \]
are orthogonal, where \( h\in \mathbb{R}^{d_\text{in}}\) is an input vector and \(\odot\) denotes element-wise multiplication.    
\end{theorem}
\begin{proof}
    Each element of $z^{(i)}$ is sampled from a Bernoulli distribution, taking only values 0 or 1. 
Consequently, for all $j$:
\[(1 - z_j^{(i)}) \cdot z_j^{(i)} = 0\]
Therefore:
\[(\mathbf{1} - z^{(i)})^T z^{(i)} = 0\]
Then, the inner product of the two vectors
    \begin{align*}
        & (m^{(j)} \odot (\Delta W_j h))^T  (m^{(i)} \odot (\Delta W_i h))\\ 
        =& \left( \left( \prod_{l=1}^{j-1} (\mathbf{1} - z^{(l)}) \right) \odot z^{(j)} \odot (\Delta W_j h)\right)^T\\
        &\left( \left( \prod_{l=1}^{i-1} (\mathbf{1} - z^{(l)}) \right) \odot z^{(i)} \odot (\Delta W_i h)\right)\\
        =& \left( (\mathbf{1} - z^{(i)}) \odot \left( \prod_{l=1, l\neq i}^{j-1} (\mathbf{1} - z^{(l)}) \right) \odot z^{(j)} \odot (\Delta W_j h) \right)^T\\
        &\left( z^{(i)} \odot \left( \prod_{l=1}^{i-1} (\mathbf{1} - z^{(l)}) \right) \odot (\Delta W_i h)\right)\\
        =& 0
    \end{align*}
    by proposition 1.
\end{proof}

\subsection{Algorithms}
\label{sec:algo}

Based on the preceding definitions and theorems, we now present the corresponding algorithms. Algorithm~\ref{alg:dropout} provides the vectorized implementation of Monte-Carlo Dropout, which directly corresponds to the formulation in (\ref{eq:monte}). Algorithm~\ref{alg:orthdropout} introduces the Orthogonal Monte-Carlo Dropout. Here, the variable \textit{acc} (short for accumulation) is used to store intermediate results, since the mask $m^{(j)}$ in (\ref{eq:mask}) includes the product $\prod_{l=1}^{j-1} (\mathbf{1} - z^{(l)})$; caching this product prevents redundant computation. Finally, Algorithm~\ref{alg:merging} describes how to merge the outputs of the pre-trained parameters and LoRAs. The function $g$ is a preprocessing transform applied to the vectors before merging (it does not perform the merge itself): for direct merging, $g$ is the identity; for Monte-Carlo Dropout, $g$ is implemented by Algorithm~\ref{alg:dropout}; for Orthogonal Monte-Carlo Dropout, $g$ is implemented by Algorithm~\ref{alg:orthdropout}.


As shown in Algorithm~\ref{alg:dropout} and Algorithm~\ref{alg:orthdropout}, the complexity of our method scales linearly with the number of merged LoRAs rather than with the input size. Since this number rarely exceeds 10 in practice, the overhead is negligible.

\begin{algorithm}[ht]
   \caption{Monte-Carlo Dropout}
   \label{alg:dropout}
\begin{algorithmic}
   \STATE {\bfseries Input:} Input vectors $\{x_1, x_2, \dots, x_k\}$, dropout probabilities $\{p_1, p_2, \dots, p_k\}$
   \STATE {\bfseries Output:} Output vectors $\{y_1, y_2, \dots, y_k\}$
   \FOR{$i = 1$ {\bfseries to} $k$}
      \STATE Sample $z_i \sim \mathrm{Bernoulli}(1 - p_i)$
      \STATE $y_i \gets z_i \odot x_i / (1 - p_i)$
   \ENDFOR
   \STATE \textbf{return} $\{y_1, y_2, \dots, y_k\}$
\end{algorithmic}
\end{algorithm}

\begin{algorithm}[ht]
   \caption{Orthogonal Monte-Carlo Dropout}
   \label{alg:orthdropout}
\begin{algorithmic}
   \STATE {\bfseries Input:} Input vectors $\{x_1, x_2, \dots, x_k\}$, dropout probabilities $\{p_1, p_2, \dots, p_k\}$ with $\sum_{i=1}^k (1 - p_i) \leq 1$
   \STATE {\bfseries Output:} Output vectors $\{y_1, y_2, \dots, y_k\}$
   \STATE Initialize $acc \gets \mathbf{1}$ \hfill // all-ones mask
   \FOR{$j = 1$ {\bfseries to} $k$}
      \IF{$j = 1$}
         \STATE Sample $z^{(1)} \sim \mathrm{Bernoulli}(1 - p_1)$
      \ELSE
         \STATE Sample $z^{(j)} \sim \mathrm{Bernoulli}\!\left(\frac{1 - p_j}{\sum_{l=1}^{j-1} p_l - (j-2)}\right)$
      \ENDIF
      \STATE $m^{(j)} \gets acc \odot z^{(j)}$
      \STATE $y_j \gets m^{(j)} \odot x_j / (1 - p_j)$
      \STATE $acc \gets acc \odot (1 - z^{(j)})$
   \ENDFOR
   \STATE \textbf{return} $\{y_1, y_2, \dots, y_k\}$
\end{algorithmic}
\end{algorithm}

\begin{algorithm}[ht]
   \caption{Merging the outputs of the pre-trained parameters and LoRAs (feed-forward)}
   \label{alg:merging}
\begin{algorithmic}
   \STATE {\bfseries Input:} Input vector $h$, pre-trained parameters $W$, LoRA weights $\Delta W_1, \dots,\Delta W_k$, merging weights $\{w_1, w_2, \dots, w_k\}$, preprocessing transform $g$. 
   \STATE {\bfseries Output:} Output vectors $\{y_1, y_2, \dots, y_k\}$

    \STATE $y_1, \dots y_k \gets g(\Delta W_1 h, \dots , \Delta W_k h)$

   \STATE \textbf{return} $Wh + \sum_{i=1}^k w_i y_i$
\end{algorithmic}
\end{algorithm}

\section{REDUNDANCY OF THE VECTORS}
\label{sec:redundancy}



In Monte-Carlo Dropout \citep{gal2016dropout}, the same dropout rate is applied during both training and inference to ensure that any subnetwork sampled at inference time converges in expectation. Theoretically, this implies that if the dropout rate used when training a LoRA matches the dropout rate applied during sampling, the generated images are guaranteed to preserve the semantics encoded by the LoRA.

Surprisingly, our empirical observations reveal that LoRA representations are highly redundant: even if no dropout (or only a very small dropout rate) is used during training, applying a large dropout rate at inference still produces images that retain the LoRA's intended semantics. This property significantly improves the generality of our method --- regardless of whether a LoRA has been trained with dropout, Monte-Carlo Dropout or Orthogonal Monte-Carlo Dropout can be applied at sampling time to merge it with other LoRAs.

We further support this claim with an empirical study presented in Section~\ref{sec:redundancyexp}.

\section{EXPERIMENTS}
\label{sec:exp}

In our experiments, we first empirically examine the LoRA redundancy discussed in Section~\ref{sec:redundancy}, and then compare three LoRA merging methods: \textbf{Direct Merge} (introduced in Equation (\ref{eq:directmerge})), \textbf{Monte Carlo Dropout} (introduced in Equation (\ref{eq:monte})), and our proposed \textbf{Orthogonal Monte Carlo Dropout}. A detailed illustration of the differences among these methods is provided in Figure~\ref{fig:figure}. We include Monte Carlo Dropout as a baseline to control for the effect of dropout itself, thereby isolating the impact of orthogonality. Without this comparison, it would be unclear whether the differences in generated images arise from dropout or from orthogonality.

\subsection{Image Generation Settings}
Unless otherwise specified, we use SDXL 1.0 \citep{podell2023sdxl} as the base model for image generation, producing images at a resolution of 1024$\times$1024 with 40 denoising steps, a CLIP skip of two, and a classifier-free guidance (CFG) scale of seven \citep{ho2022classifier}. Since diffusion-based image generation is highly sensitive to the initial noise \citep{guo2024initno, miao2025noise}, we fix the random seed for the initial noise across comparisons to \textbf{rigorously isolate the effects of orthogonality}.

\subsection{Preparing LoRAs}
To control experimental variables, we trained a separate LoRA for each concept in the DreamBooth dataset using identical hyperparameters. 

In addition, to demonstrate that the redundancy we highlight is a general property of LoRAs and to showcase the broad applicability of our method, we also collected a variety of pre-trained LoRAs from the community that were trained under different settings.

\subsubsection{Training LoRAs based on Dreambooth Datasets}
We use the 30 concepts provided in the DreamBooth dataset \citep{ruiz2023dreambooth}, training a separate LoRA for each concept. Each LoRA is trained with a rank of 128 using the AdamW optimizer, with a learning rate of $10^{-4}$ for the UNet and $10^{-5}$ for the text encoder. A training dropout rate of $0.05$ is applied to ensure stability.

\subsubsection{Obtaining LoRAs from the Community}
To demonstrate the generality of the redundancy and the broad applicability of our method, we also collected LoRAs from model-sharing platforms such as CivitAI and HuggingFace. Some of these LoRAs include metadata specifying their training parameters, while others do not. For detailed information, please refer to the Github repository associated with this paper\footnote{\url{https://github.com/andiac/Orth_MC_Dropout}}.

\subsection{Demonstrating the Redundancy of LoRAs}
\label{sec:redundancyexp}

We demonstrate the redundancy of LoRAs by comparing images generated from the same LoRA under different sampling dropout rates. As shown in Figure~\ref{fig:redundancydb} and Figure~\ref{fig:redundancycommunity}, varying the sampling dropout rate has little effect on the preservation of LoRA semantics. Even at a very high dropout rate (0.9)\footnote{At this rate, Orthogonal Monte-Carlo Dropout supports the merging of up to 10 LoRAs.}, the semantics encoded by the original LoRA remain well preserved, although image quality degrades noticeably. In Figure~\ref{fig:redundancydb}, all LoRAs are trained on the DreamBooth dataset with a fixed training dropout rate of 0.05. By contrast, the LoRAs in Figure~\ref{fig:redundancycommunity} are collected from the community and trained under diverse parameter settings—for example, the first row (Kimoju, a cartoon-style character) was trained with a dropout rate of 0.9, while the second row (Keqing from Genshin Impact) was trained without dropout. Despite these differences, both controlled and community-sourced LoRAs consistently exhibit strong semantic preservation across a wide range of sampling dropout rates.

\begin{figure}[tb]
  \begin{center}
      \input{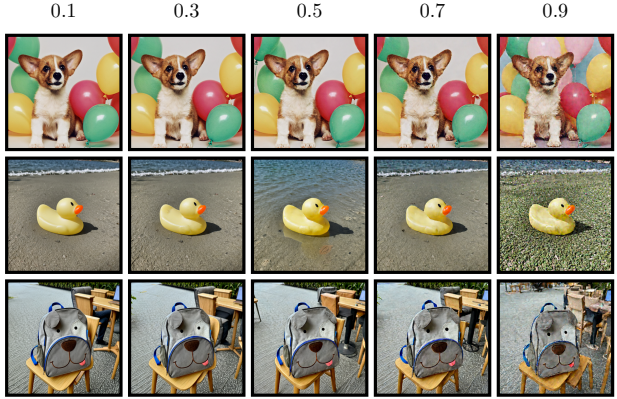}      
  \end{center}
  \caption{Empirical study of Monte-Carlo Dropout under different dropout rates (0.1 - 0.9). All LoRAs are trained on Dreambooth dataset with a training dropout rate of 0.05.}
  \label{fig:redundancydb}
\end{figure}

\begin{figure}[tb]
  \begin{center}
      \input{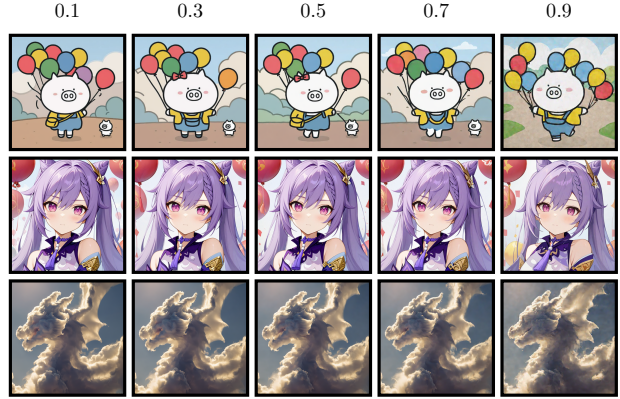}
  \end{center}
  \caption{Empirical study of Monte-Carlo Dropout under different sampling dropout rates (0.1–0.9) using LoRAs downloaded from the community. According to the metadata, the first row (cartoon character Kimoju) was trained with a dropout rate of 0.9, the second row (Keqing from Genshin Impact) was trained without dropout, and the third row (cloud style LoRA) does not disclose its training dropout rate.}
  \label{fig:redundancycommunity}
\end{figure}

\subsection{Merging LoRAs Trained on Dreambooth Datasets}
Pairwise merging is carried out independently under each of the three described LoRA merging methods: \textbf{Direct Merge}, \textbf{Monte Carlo Dropout}, and our proposed \textbf{Orthogonal Monte Carlo Dropout}, followed by evaluation of the generated images. For each pairwise combination, we set $w_1 = w_2 = 1$ and $p_1 = p_2 = 0.5$, and generate four images, yielding a total of $(30 \times 29)/2 \times 4 = 1740$ images. As shown in Figure~\ref{fig:dreambooth}, clear semantic differences among the three methods are difficult to discern by visual inspection, though we observe a noticeable decline in image quality for both Monte Carlo Dropout and Orthogonal Monte Carlo Dropout. To quantify our observations, we use the CLIP score \citep{hessel2021clipscore} to measure semantic similarity, conduct an A/B test user study to compare the fidelity of concept representation, and further assess the quality of the generated images.

In the user study, we evaluated a total of $1740 \times 3 = 5220$ cases. Each case presented the reference images of two concepts, followed by two generated images from the compared methods. Participants were asked to select which image better represented both concepts. Note that participants received no prior training; they were simply asked to make a subjective choice based on their perception. This design reflects the fact that ``semantic fusion'' and ``image quality'' are difficult to define precisely. Following the works of \citet{song2018constructing} and \citet{zhang2024constructing}, we recruited five participants and determined the outcome for each case by majority vote. When both methods received approximately 50\% of the votes, we considered the difference indistinguishable to participants.


Table~\ref{table:similarity} reports pairwise comparisons among the three merging strategies. CLIP-based semantic alignment is uniformly high across all pairs---0.9116 (Direct vs. Dropout), 0.9154 (Direct vs. Orthogonal), and 0.9346 (Dropout vs. Orthogonal)---indicating that all methods preserve concept semantics at a similar level (CLIP $>$ 0.91). The user study complements this result with preference rates based on majority vote: participants preferred Direct over Dropout (84.0\%) and Direct over Orthogonal (81.3\%), whereas Dropout vs. Orthogonal was close to parity (47.3\% vs. 52.7\%). Together, these findings suggest that perceived differences are driven primarily by image quality rather than semantics, with the near-parity between Dropout and Orthogonal consistent with their similar semantic alignment. 

Table~\ref{table:imgquality} reports five no-reference image quality assessment (IQA) metrics---MUSIQ, ARNIQA, NIMA, TRes, and DBCNN---where higher is better. Across all metrics, Direct yields the best scores, while Orthogonal is consistently (slightly) better than Dropout. This quality ordering---Direct $>$ Orthogonal $>$ Dropout---helps explain the user study: the strong preference for Direct over the other two methods and the slight edge for Orthogonal over Dropout (47.3\% vs. 52.7\%) are aligned with their measured image-quality gap.

\begin{figure}[tb]
\begin{center}
    \input{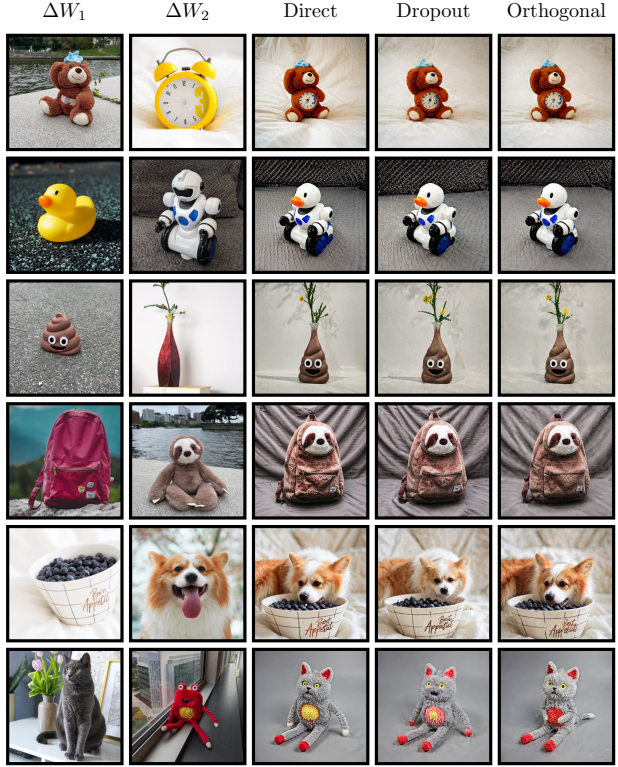}    
\end{center}
\caption{Examples of DreamBooth LoRA merges. `Direct' denotes Direct Merge, `Dropout' denotes Monte Carlo Dropout, and `Orthogonal' denotes Orthogonal Monte Carlo Dropout.}
\label{fig:dreambooth}
\end{figure}

\begin{figure}[hbt]
\begin{center}
    \input{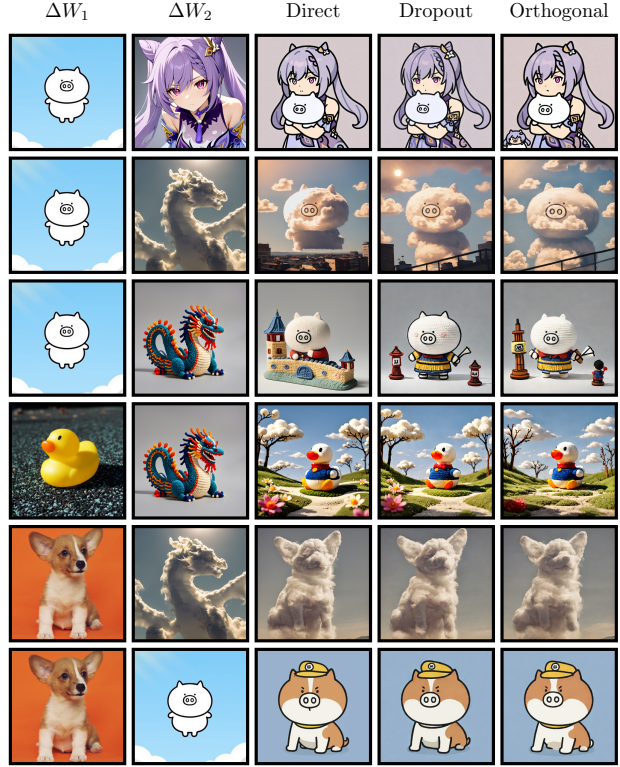}    
\end{center}
\caption{Merging LoRAs from the community. `Direct' denotes Direct Merge, `Dropout' denotes Monte Carlo Dropout, and `Orthogonal' denotes Orthogonal Monte Carlo Dropout.}
\label{fig:community}
\end{figure}

\begin{table}[hbp]
\caption{Quantatitive result on similarity and fidelity.}
\label{table:similarity}
\begin{center}
\setlength{\tabcolsep}{3.5pt}
\renewcommand{\arraystretch}{1.1}
\begin{adjustbox}{max width=\linewidth}
\begin{tabular}{lcccccc}
\toprule
\textbf{} & \textbf{Direct vs Dropout} & \textbf{Direct vs Orthogonal} & \textbf{Dropout vs Orthogonal} \\
\midrule
CLIP & 0.9116 & 0.9154 & 0.9346 \\
User Study  & 84.0\% & 81.3\% & 47.3\% \\
\bottomrule
\end{tabular}
\end{adjustbox}
\end{center}
\end{table}

\subsection{Merging LoRAs from Community Sources}
In addition to the controlled quantitative experiments presented in the previous section, Figure~\ref{fig:community} shows results from merging LoRAs downloaded from the community. Since our merging operates in the hidden space, LoRAs with different ranks can also be freely combined. These results demonstrate that our proposed method requires no additional fine-tuning and generalizes well, while further highlighting that orthogonality plays a limited role in LoRA merging.

\begin{table}[hbtp]
  \caption{Image Quality.}
  \label{table:imgquality}
  \begin{center}
  \setlength{\tabcolsep}{3.5pt}
  
  
  \renewcommand{\arraystretch}{1.1}
  \begin{adjustbox}{max width=\linewidth}
  \begin{tabular}{lcccccc}
  \toprule
  \textbf{} & \textbf{Direct} & \textbf{Dropout} & \textbf{Orthogonal} \\
  \midrule
  $\uparrow$ MUSIQ \citep{ke2021musiq} & \textbf{62.8063} & 59.5810 & 59.6875 \\
  $\uparrow$ ARNIQA \citep{agnolucci2024arniqa}  & \textbf{0.7529}  & 0.7398 & 0.7423 \\
  $\uparrow$ NIMA \citep{talebi2018nima} & \textbf{4.2824} & 4.1120  & 4.1148 \\
  $\uparrow$ TRes \citep{golestaneh2022no}  & \textbf{87.3702} & 84.0270  & 84.1839 \\
  $\uparrow$ DBCNN \citep{zhang2020blind}  & \textbf{0.6647} & 0.6419  & 0.6427 \\
  \bottomrule
  \end{tabular}
  \end{adjustbox}
  \end{center}
  \end{table}




\section{DISCUSSIONS}

\subsection{Merging LoRA weights or merging LoRA outputs?}  
Throughout this paper, we have referred to ``merging LoRA modules.'' In practice, however, our method operates by merging the \textbf{outputs} of LoRA modules. For clarity, note that directly merging the \textbf{weights} of LoRAs via a weighted sum is mathematically equivalent to merging their outputs, due to the linearity of matrix multiplication:  
$$(w_1 \Delta W_1 + w_2 \Delta W_2)h = w_1 \Delta W_1 h + w_2 \Delta W_2 h.$$  

\subsection{Controlled Variables}  
In the three methods illustrated in Figure~\ref{fig:figure}, one key controlled variable is the expected scale of the combined vectors. As shown in (\ref{eq:monte}), although $z \sim \text{Ber}(1-p)$ reduces the expected scale by a factor of $(1-p)$, dividing by $(1-p)$ restores the expectation to 1. Consequently, as demonstrated in our redundancy experiments, even with dropout masking 90\% of the LoRA activations, controlling the expected scale ensures that semantic preservation remains intact.  

Another critical controlled variable is the \textbf{initial noise} in diffusion-based image generation. As noted by \citet{guo2024initno} and \citet{miao2025noise}, the initial noise has a substantial impact on the final image. To the best of our knowledge, our study is the only one in this line of work that strictly controls the initial noise, which is essential for isolating and evaluating the true effect of orthogonality on semantic preservation.

\subsection{Interference as a Source of Image Quality Gains}

Direct Merge achieves substantially higher image quality than the other methods, indicating that the interactions arising from directly adding LoRAs can, in fact, enhance visual quality. This finding suggests that interference between LoRA modules is not always harmful; in some cases, it may introduce constructive interactions that improve perceptual fidelity. Such an outcome challenges the prevailing assumption that enforcing orthogonality is inherently more advantageous for merging. Instead, our results highlight a more nuanced trade-off: while orthogonality helps eliminate interference at the representational level, it may simultaneously suppress beneficial interactions that contribute to visual quality. This observation motivates a re-examination of the role of orthogonality in adapter merging, suggesting that effective strategies may need to balance interference reduction with the potential gains brought by controlled interactions.

\subsection{User Study: Isolating Orthogonality}

While participants generally preferred higher-quality images, the Dropout vs. Orthogonal comparison in Table~\ref{table:similarity} shows that when image quality is comparable, enforcing strict orthogonality provides no observable benefit in user-perceived concept fidelity.

\subsection{Relationship with Zip-LoRA}

We do not dispute the effectiveness of Zip-LoRA \citep{shah2023ziplora}; however, we note that its performance is more likely attributable to the additional fine-tuning it requires rather than to orthogonality itself. Importantly, Zip-LoRA does not enforce strict orthogonality. Instead, it merely encourages orthogonality by introducing an inner-product penalty into the loss function, which provides no theoretical guarantee of orthogonal representations.

\subsection{What Do People Expect from Orthogonality?}  
Orthogonality is an elegant mathematical property, but its semantic implications are often vaguely defined. For example, should orthogonality lead to better \textbf{disentanglement} of semantics, or to better \textbf{fusion} of semantics? Although these notions lack precise mathematical definitions, disentanglement and fusion are often seen as contradictory. Thus, claims such as ``orthogonality improves disentanglement'' or ``orthogonality improves fusion'' are largely intuitive rather than rigorously grounded. Our findings demonstrate that, in the context of LoRA merging, orthogonality provides no tangible benefit. This suggests that \textbf{the community should adopt a more cautious perspective when attributing semantic advantages to the mathematical property of orthogonality.}


\section{CONCLUSION}  
We introduced Orthogonal Monte Carlo Dropout to enforce strict orthogonality in LoRA merging without extra computational cost. While our theoretical analysis confirms runtime orthogonality and our experiments show the redundancy of LoRAs across both controlled and community settings, the results reveal a clear negative finding: orthogonality alone brings little benefit, as it neither improves semantic disentanglement nor compositionality, and in some cases even reduces image quality compared to direct merging. This suggests that the field should rethink the assumed semantic advantages of orthogonality in adapter merging and explore alternative strategies that balance interference reduction with the preservation of constructive interactions.

\newpage

\bibliography{ref}

\begin{thebibliography}{}

\bibitem[Agnolucci et~al., 2024]{agnolucci2024arniqa}
Agnolucci, L., Galteri, L., Bertini, M., and Del~Bimbo, A. (2024).
\newblock Arniqa: Learning distortion manifold for image quality assessment.
\newblock In {\em Proceedings of the IEEE/CVF Winter Conference on Applications of Computer Vision}, pages 189--198.

\bibitem[Gal and Ghahramani, 2016]{gal2016dropout}
Gal, Y. and Ghahramani, Z. (2016).
\newblock Dropout as a bayesian approximation: Representing model uncertainty in deep learning.
\newblock In {\em international conference on machine learning}, pages 1050--1059. PMLR.

\bibitem[Golestaneh et~al., 2022]{golestaneh2022no}
Golestaneh, S.~A., Dadsetan, S., and Kitani, K.~M. (2022).
\newblock No-reference image quality assessment via transformers, relative ranking, and self-consistency.
\newblock In {\em Proceedings of the IEEE/CVF winter conference on applications of computer vision}, pages 1220--1230.

\bibitem[Guo et~al., 2024]{guo2024initno}
Guo, X., Liu, J., Cui, M., Li, J., Yang, H., and Huang, D. (2024).
\newblock Initno: Boosting text-to-image diffusion models via initial noise optimization.
\newblock In {\em Proceedings of the IEEE/CVF Conference on Computer Vision and Pattern Recognition}, pages 9380--9389.

\bibitem[Hessel et~al., 2021]{hessel2021clipscore}
Hessel, J., Holtzman, A., Forbes, M., Bras, R.~L., and Choi, Y. (2021).
\newblock Clipscore: A reference-free evaluation metric for image captioning.
\newblock {\em arXiv preprint arXiv:2104.08718}.

\bibitem[Ho and Salimans, 2022]{ho2022classifier}
Ho, J. and Salimans, T. (2022).
\newblock Classifier-free diffusion guidance.
\newblock {\em arXiv preprint arXiv:2207.12598}.

\bibitem[Hu et~al., 2022]{hu2022lora}
Hu, E.~J., Shen, Y., Wallis, P., Allen-Zhu, Z., Li, Y., Wang, S., Wang, L., Chen, W., et~al. (2022).
\newblock Lora: Low-rank adaptation of large language models.
\newblock {\em ICLR}, 1(2):3.

\bibitem[Huang et~al., 2023]{huang2023lorahub}
Huang, C., Liu, Q., Lin, B.~Y., Pang, T., Du, C., and Lin, M. (2023).
\newblock Lorahub: Efficient cross-task generalization via dynamic lora composition.
\newblock {\em arXiv preprint arXiv:2307.13269}.

\bibitem[Ke et~al., 2021]{ke2021musiq}
Ke, J., Wang, Q., Wang, Y., Milanfar, P., and Yang, F. (2021).
\newblock Musiq: Multi-scale image quality transformer.
\newblock In {\em Proceedings of the IEEE/CVF international conference on computer vision}, pages 5148--5157.

\bibitem[Miao et~al., 2025]{miao2025noise}
Miao, B., Li, C., Wang, X., Zhang, A., Sun, R., Wang, Z., and Zhu, Y. (2025).
\newblock Noise diffusion for enhancing semantic faithfulness in text-to-image synthesis.
\newblock In {\em Proceedings of the Computer Vision and Pattern Recognition Conference}, pages 23575--23584.

\bibitem[Po et~al., 2024]{po2024orthogonal}
Po, R., Yang, G., Aberman, K., and Wetzstein, G. (2024).
\newblock Orthogonal adaptation for modular customization of diffusion models.
\newblock In {\em Proceedings of the IEEE/CVF conference on computer vision and pattern recognition}, pages 7964--7973.

\bibitem[Podell et~al., 2023]{podell2023sdxl}
Podell, D., English, Z., Lacey, K., Blattmann, A., Dockhorn, T., M{\"u}ller, J., Penna, J., and Rombach, R. (2023).
\newblock Sdxl: Improving latent diffusion models for high-resolution image synthesis.
\newblock {\em arXiv preprint arXiv:2307.01952}.

\bibitem[Prabhakar et~al., 2024]{prabhakar2024lora}
Prabhakar, A., Li, Y., Narasimhan, K., Kakade, S., Malach, E., and Jelassi, S. (2024).
\newblock Lora soups: Merging loras for practical skill composition tasks.
\newblock {\em arXiv preprint arXiv:2410.13025}.

\bibitem[Ruiz et~al., 2023]{ruiz2023dreambooth}
Ruiz, N., Li, Y., Jampani, V., Pritch, Y., Rubinstein, M., and Aberman, K. (2023).
\newblock Dreambooth: Fine tuning text-to-image diffusion models for subject-driven generation.
\newblock In {\em Proceedings of the IEEE/CVF Conference on Computer Vision and Pattern Recognition}, pages 22500--22510.

\bibitem[Shah et~al., 2023]{shah2023ziplora}
Shah, V., Ruiz, N., Cole, F., Lu, E., Lazebnik, S., Li, Y., and Jampani, V. (2023).
\newblock Ziplora: Any subject in any style by effectively merging loras.
\newblock {\em arXiv preprint arXiv:2311.13600}.

\bibitem[Song et~al., 2018]{song2018constructing}
Song, Y., Shu, R., Kushman, N., and Ermon, S. (2018).
\newblock Constructing unrestricted adversarial examples with generative models.
\newblock {\em Advances in neural information processing systems}, 31.

\bibitem[Srivastava et~al., 2014]{srivastava2014dropout}
Srivastava, N., Hinton, G., Krizhevsky, A., Sutskever, I., and Salakhutdinov, R. (2014).
\newblock Dropout: a simple way to prevent neural networks from overfitting.
\newblock {\em The journal of machine learning research}, 15(1):1929--1958.

\bibitem[Talebi and Milanfar, 2018]{talebi2018nima}
Talebi, H. and Milanfar, P. (2018).
\newblock Nima: Neural image assessment.
\newblock {\em IEEE transactions on image processing}, 27(8):3998--4011.

\bibitem[Wang et~al., 2024]{wang2024lora}
Wang, H., Ping, B., Wang, S., Han, X., Chen, Y., Liu, Z., and Sun, M. (2024).
\newblock Lora-flow: Dynamic lora fusion for large language models in generative tasks.
\newblock {\em arXiv preprint arXiv:2402.11455}.

\bibitem[Zhang et~al., 2024]{zhang2024constructing}
Zhang, A., Zhang, M., and Wischik, D. (2024).
\newblock Constructing semantics-aware adversarial examples with a probabilistic perspective.
\newblock {\em Advances in Neural Information Processing Systems}, 37:136259--136285.

\bibitem[Zhang et~al., 2020]{zhang2020blind}
Zhang, W., Ma, K., Yan, J., Deng, D., and Wang, Z. (2020).
\newblock Blind image quality assessment using a deep bilinear convolutional neural network.
\newblock {\em IEEE Transactions on Circuits and Systems for Video Technology}, 30(1):36--47.

\bibitem[Zhong et~al., 2024]{zhong2024multi}
Zhong, M., Shen, Y., Wang, S., Lu, Y., Jiao, Y., Ouyang, S., Yu, D., Han, J., and Chen, W. (2024).
\newblock Multi-lora composition for image generation.
\newblock {\em arXiv preprint arXiv:2402.16843}.

\end{thebibliography}







\end{document}